\documentclass[11pt,a4paper]{article}
\usepackage[hyperref]{emnlp2020}
\usepackage{times}
\usepackage{latexsym}

\usepackage{microtype}
\usepackage{amsthm,amsmath,amssymb,amsfonts,bbm}
\usepackage{graphicx}
\usepackage{xspace}
\usepackage{xcolor,colortbl}

\newcommand{\parfrac}[2]{\paran{\frac{#1}{#2}}}
\newcommand{\paran}[1]{\left( #1 \right)}

\usepackage{balance}
\usepackage{booktabs} 
\usepackage{graphicx}
\usepackage{subfig}
\usepackage{multirow}
\graphicspath{{./imgs/}}

\usepackage{enumitem}

\usepackage[ruled,vlined,linesnumbered]{algorithm2e}
\newcommand{\gmbl}{\mathsf{Gumbel}}
\newcommand{\gumbel}[2]{\gmbl\paran{#1,#2}}
\newcommand{\trunPoisson}{\mathsf{TruncatedPoisson}}
\newcommand{\gumbelMech}{\textsc{Truncated-Gumbel-Arg-Min}}
\newcommand{\poisson}[1]{\trunPoisson\paran{#1}}
\newcommand{\gumbelArgMin}[3]{\gumbelMech\paran{#1,#2,#3}}

\newtheorem{definition}{Definition}
\newtheorem{lemma}{Lemma}
\newtheorem{theorem}{Theorem}

\aclfinalcopy
\title{Differentially Private Adversarial Robustness Through Randomized Perturbations}

\author{Nan Xu$^\$$, Oluwaseyi Feyisetan$^*$, Abhinav Aggarwal$^*$, Zekun Xu$^*$, Nathanael Teissier$^\dagger$\\$^\$$University of Southern California, Los Angeles, CA, USA\\$^*$Amazon Alexa, Seattle, WA, USA\\$^\dagger$Amazon Alexa, Arlington, VA, USA\\\texttt{nanx@usc.edu, \{sey,aggabhin,zeku,natteis\}@amazon.com}}

\date{}

\begin{document}

\maketitle
\begin{abstract}
Deep Neural Networks, despite their great success in diverse domains, are provably sensitive to small perturbations on correctly classified examples and lead to erroneous predictions. Recently, it was proposed that this behavior can be combatted by optimizing the worst case loss function over all possible substitutions of training examples. However, this can be prone to weighing unlikely substitutions higher, limiting the accuracy gain. In this paper, we study adversarial robustness through randomized perturbations, which has two immediate advantages: (1) by ensuring that substitution likelihood is weighted by the proximity to the original word, we circumvent optimizing the worst case guarantees and achieve performance gains; and (2) the calibrated randomness imparts differentially-private model training, which 
%prevents leakage of sensitive information %REMOVED AS PER REQUEST FROM TOM T VIA PRAKASH
% suggested replacement was: “improve robustness against adversarial attacks” TODO: re-read
additionally improves robustness against adversarial attacks on 
%from 
the model outputs. Our approach uses a novel density-based mechanism based on truncated Gumbel noise, which ensures training on substitutions of both rare and dense words in the vocabulary while maintaining semantic similarity for model robustness.

% Through randomized perturbations, differential privacy (DP) was proposed and established as the state-of-the-art standard to prevent leaks of private information in data.
% In textual analysis, we propose to generate adversarial examples with metric differential mechanisms, one variant of DP where words close to the original one are assigned higher probability as substitutions during word perturbations. We also investigate robustness of different adversarial training approaches against these randomized perturbations. Motivated by the limitation of the existing multivariate laplace mechanism in semantic-preserving, we introduce a novel truncated gumbel perturbation mechanism, which narrows down the range of substitution candidates in advance to better preserve semantic meanings hence help maintain utility of downstream models trained on perturbed texts. 
\end{abstract}
\section{Introduction}
%\abhi{Needs fixing.} 
Deep neural networks (DNNs) have found applications within multiple domains: from computer vision~\cite{krizhevsky2012imagenet}, and natural language processing~\cite{mikolov2013distributed}, to robotics~\cite{kober2013reinforcement} and self-driving cars~\cite{bojarski2016end}. However, DNNs have been shown to be vulnerable to adversarial examples. These are small perturbations of examples that are correctly classified by well-trained models but incorrectly classified in the target~\cite{szegedy2013intriguing,goodfellow2014explaining}. 
%and healthcare~\cite{jiang2017artificial}
%Unlike perturbations which are virtually indistinguishable to human perception in the image domain, small perturbations applied to texts are clearly perceptible but still result in semantically and syntactically similar adversarial examples~\cite{jia2017adversarial,alzantot2018generating,iyyer2018adversarial,ribeiro2018semantically,belinkov2017synthetic,ebrahimi2017hotflip}.

A few approaches have been proposed to defend against such adversarial attacks. One of the most widely used methods is adding the adversarial examples to the original training set and retraining the model. On most kinds of perturbations, such augmented training approach has achieved improved robustness without harming accuracy on the original testing sets~\cite{jia2017adversarial,iyyer2018adversarial,ribeiro2018semantically,belinkov2017synthetic,ebrahimi2017hotflip}. However, this often leads to the augmented neural network over-fitting to the additional data~\cite{matyasko2017margin}, but failing to perform robustly against other types of adversarial examples~\cite{jia2017adversarial,belinkov2017synthetic}. Recently, certified defences have been adopted in the computer vision domain~\cite{lecuyer2019certified,dvijotham2018training,gowal2018effectiveness}. To defend against perturbations on text data, the Interval Bounded Propagation (IBP) approach was proposed by~\cite{jia2019certified} to minimize the upper bound on the worst-case loss that word substitutions can induce during the training procedure.

In this paper, we propose a new approach to generate adversarial examples via word substitutions in textual analysis. Our approach is based on randomized mechanisms satisfying Metric Differential Privacy ($d_{\chi}$-privacy~\cite{andres2013geo}) -- a variant of Differential privacy (DP). DP was proposed by ~\cite{dwork2006calibrating} and has been established as a \textit{de facto} standard for privacy-preserving data analysis. It mathematically guarantees, given a privacy parameter $\epsilon$, that an adversary observing separate outputs of computations over adjacent databases (described by a Hamming distance) will make essentially the same inference. As opposed to standard DP, with $d_{\chi}$-privacy, the guarantees are scaled by a (different) distance metric between adjacent databases, and privacy preserving noise is sampled from a multivariate (Laplacian) distribution. The distances are over a metric space as defined by word embeddings such as GloVe~\cite{pennington2014glove} or fastText~\cite{bojanowski2017enriching}, while the data points are vector representions of the words. The mechanism assigns higher substitution probability, based on the noise added, to words closer to the original one than those further away. The private text mechanisms proposed by~\cite{fernandes2019generalised} and~\cite{feyisetan2019leveraging,feyisetan2020privacy} work using this approach.

However, for words with embedding vectors in dense areas, the existing multivariate Laplace mechanisms fail to distinguish nearer (\emph{i.e.}, more relevant) words from other close but less relevant words. As a result, for a given value of the privacy parameter $\epsilon$, an irrelevant word could have a similar substitution probability as a relevant word. We propose a new metric-DP mechanism called the truncated Gumbel perturbation mechanism to allow a smaller range of nearby words considered than the multivariate Laplace mechanism. The new mechanism samples a $k$ value from a truncated Poisson distribution as substitution candidates before perturbation, hence words nearby with irrelevant meanings are disregarded. This better preserves word semantics and improves utility of models trained on perturbed datasets in downstream tasks.

%The proposed approach to generate adversarial examples is quite different from those introduces in literatures~\cite{jia2017adversarial,alzantot2018generating,iyyer2018adversarial,ribeiro2018semantically,belinkov2017synthetic,ebrahimi2017hotflip}, as similar semantic meanings are not guaranteed all the time in the perturbed texts. Instead, the level of semantic-preserving is determined by the privacy parameter $\epsilon$ in the metric DP mechanisms: larger values of $\epsilon$ result in smaller noise and similar word substitutions while smaller values tend to pick irrelevant words as substitutions. \zekun{Seems like we need a new paragraph starting from the next sentence because it is no longer about truncated gumbel. The transition here is from privacy defense to robust training.}

%The training procedure of the existing certified defence model IBP~\cite{jia2019certified} requires that the substitution candidates of all words in the vocabulary are semantically similar and known in advance. 

In this paper, we investigate the performance of a well-trained IBP model on classification tasks when the input text is perturbed by a metric DP mechanism with different values of $\epsilon$ -- corresponding to different degrees of semantic preservation. Motivated by the success of augmented training with adversarial data such as~\cite{jia2017adversarial}, we also add the adversarial examples generated by the privacy mechanisms to the original training set while comparing its robustness with IBP. 

The contributions of this paper is as follows:
\begin{itemize}
\item We propose a novel metric-DP mechanism called the truncated Gumbel mechanism, which provides formal privacy guarantees, and better preserves semantic meanings than the existing multivariate Laplace mechanisms.
\item To the best of our knowledge, we are the first to leverage metric-DP mechanisms to generate adversarial examples and study the performance of different adversarial training approaches at different values of $\epsilon$. 
\item We empirically demonstrate the benefit of the truncated Gumbel mechanism in preserving semantics and show that augmented training performs better than certifiably robust training, both in clean and adversarial accuracy.
\end{itemize}
\section{Related Work}
\paragraph{Privacy Preservation}
% Artificial intelligence (AI) is now showing its strengths in almost every domain~\cite{krizhevsky2012imagenet,mikolov2013distributed,jiang2017artificial,kober2013reinforcement,bojarski2016end}, while the big data-driven AI is difficult to be realized in all aspects of our lives as expected~\cite{yang2019federated}. One of the important reasons is data privacy issues, where many data owners are prevented by privacy and confidentiality concerns from sharing the data and thus benefiting from large-scale deep learning~\cite{shokri2015privacy,yang2019federated}. 
DP~\cite{dwork2006calibrating} preserves privacy on the output of a computation by adding noise sampled from a certain distribution (e.g. Laplace). The magnitude of the noise is proportional to the \emph{sensitivity} of the computation, and controlled by the parameter $\epsilon$. 
%By introducing exactly as much noise as is necessary to combat attacks, Differential Privacy (DP) was proposed and established as a \textit{de facto} standard for privacy-preserving data analysis algorithms~\cite{dwork2006calibrating}. 
% Dwork proposed the well-known Gaussian mechanism as the solution to DP in 2006~\cite{dwork2006our} and further improved it in 2014~\cite{dwork2014algorithmic}. Recently, Durfee \emph{et al.} added noises sampled from Gumbel distribution to the histogram counts for top-k selection over a large domain universe subject to user-level differential privacy~\cite{durfee2019practical}.
% However, it's impractical to adopt DP into privacy-preserving textual analysis directly as it requires that the given word $x$ has non-negligible probability of being transformed into any other word $x'$, regardless of how unrelated $x$ and $x'$ are. 
We consider a relaxation of DP, metric DP or $d_{\chi}$-privacy, that originated in the context of location privacy, where locations close to the user are assigned higher probability those far away~\cite{andres2013geo,chatzikokolakis2013broadening}.
% Given a procedure or mechanism $M$ that takes a sensitive dataset $x$ and releases the output $M(x)$, we compare the output $M(x)$ with a hypothetical output $M(x')$ in which the distance between input $x$ and $x'$ us denoted by metric $d(x, x')$. 
% The requirement of $d_{\chi}$-privacy is that the indistinguishability of the output distribution should be scaled by the distance $d(x, x')$ between input $x$ and $x'$, hence more about the input $x$ is allowed to remember~\cite{feyisetan2020privacy}.
% While Hamming distance is utilized in DP to describe single individual difference in databases, privacy-preserving approaches accord with $d_{\chi}$-privacy have been proposed in different distance metrics, including Euclidean~\cite{chatzikokolakis2013broadening,feyisetan2020privacy}, Manhattan~\cite{chatzikokolakis2015constructing}, Chebyshev~\cite{wagner2018technical} and Hyperbolic~\cite{feyisetan2019leveraging} metrics. 
For text, the corollary to geo-location cooridinates are word vectors in an embedding space. To preserve privacy, noise is sampled from a multivariate distribution such as the multivariate Laplace mechanism in ~\cite{fernandes2019generalised,feyisetan2020privacy} or a hyperbolic distribution in~\cite{feyisetan2019leveraging}.

%As with geo-locations in $2d$ space, for words represented in a continuous Euclidean space such as with word embeddings, we adopt the Euclidean distance metric to reflect the semantic similarity between words. However, privacy preserving noise is sampled from the multivariate Laplace mechanism on word representations~\cite{feyisetan2020privacy}. 
%Besides, we propose a novel truncated Gumbel perturbation mechanism which also satisfies metric DP but better preserve semantic meanings.

\paragraph{Adversarial Attacks} 
Deep neural networks are vulnerable to adversarial examples, where perturbations applied to examples correctly classified by well-trained models, lead to mis-classification with significantly high confidence~\cite{szegedy2013intriguing,goodfellow2014explaining}. In the text domain, adversarial example generation includes techniques for extraneous text insertion~\cite{jia2017adversarial}, word substitution~\cite{alzantot2018generating}, paraphrasing~\cite{iyyer2018adversarial,ribeiro2018semantically}, and character-level noise~\cite{belinkov2017synthetic,ebrahimi2017hotflip}. In this paper, we generate adversarial examples by word-level perturbations without semantic-preservation constraints. Specifically, randomized perturbations satisfying metric-DP are employed, with the privacy parameter $\epsilon$ controlling semantic similarity during substitutions.

\paragraph{Adversarial Training}
Augmenting training sets with adversarial examples is a common way of improving robustness in adversarial training~\cite{szegedy2013intriguing,goodfellow2014explaining}. Although it achieves improved robustness without harming accuracy on the original testing sets~\cite{jia2017adversarial,iyyer2018adversarial,ribeiro2018semantically,belinkov2017synthetic,ebrahimi2017hotflip}, augmented training is still vulnerable when tested on other adversarial examples~\cite{jia2017adversarial,belinkov2017synthetic}.
% Given the clean and perturbed data sets, a regularization term for adversarial examples has been introduced to trade adversarial robustness off against accuracy via surrogate-loss minimization~\cite{zhang2019theoretically}. Besides, different word representations like the average over character embeddings~\cite{belinkov2017synthetic} or adding noise to original embeddings~\cite{ebrahimi2017hotflip} have been tried to reduce effects of character-level perturbations.
Certified defences which provide guarantees of robustness to norm-bounded attacks have become popular in computer vision~\cite{lecuyer2019certified,dvijotham2018training,gowal2018effectiveness}. For text, the Interval Bound Propagation (IBP) approach minimizes an upper bound on the worst-case loss during training that any combination of word substitutions can induce~\cite{jia2019certified}. This requires that the allowed word substitutions are known \emph{a-priori}.
In this paper, we study the robustness of an IBP-trained model on adversarial examples generated by metric DP mechanisms. Furthermore, we analyze how adding adversarial examples into the training set can help improve robustness.

\paragraph{Connections between Privacy Preservation and Adversarial Learning}
To the best of our knowledge, this paper is the first to propose: perturbing text with metric-DP mechanisms, and testing the robustness of adversarial training approaches with these adversarial examples. Connections between privacy and adversarial learning have been studied extensively in the different domains~\cite{pinot2019unified}.
% By abstracting the definitions of differential privacy and robustness to adversarial examples, the two notions have shown to build upon the same theoretical ground and results obtained so far in one domain can be transferred to the other~\cite{pinot2019unified}. 
Two key properties of DP have been leveraged to add a noise layer to the network's architecture to provide guaranteed robustness against adversarial examples~\cite{lecuyer2019certified}. Similarly, trade-offs between DP preservation and provable robustness have been studied by learning private model parameters first followed by rigorous robustness bound computation~\cite{phan2019scalable,phan2019heterogeneous}. 
% \paragraph{Other Tasks} Differential-private learning procedures are widely used as tools to help improve performance on a wide variety of tasks. Differential Privacy was considered as a denfense measure against data poisoning attacks where the training set is adversarially modified~\cite{ma2019data}. Both theoretical and empirical analysis on how differential privacy helps to improve utility of outlier detection and novelty detection was provided in~\cite{du2019robust}.
\section{Technical Preliminaries}
We begin with providing some background on metric Differential Privacy and the multivariate Laplace mechanism, which is commonly used in privacy-preserving textual analysis.
\paragraph{Differential Privacy}
First proposed by~\cite{dwork2006calibrating}, DP provides a strong mathematical framework for guaranteeing that the output of a randomized mechanism will remain essentially unchanged on any two neighboring input databases. Formally, a randomized mechanism $M:\mathcal{X}\rightarrow\mathcal{Y}$ satisfies $(\epsilon,\delta)$-DP if for any $x,x'\in\mathcal{X}$ that differ in only one entry, then it holds for all $Y\subseteq\mathcal{Y}$ that:
\begin{align}
\text{Pr}[M(x)\in Y]\leq e^{\epsilon}\text{Pr}[M(x')\in Y]+\delta,
\end{align}
where $\epsilon>0$ and $\delta\in[0,1]$ are parameters that quantify the strength of the privacy guarantee. If $\delta=0$, we say that the mechanism $M$ is $\epsilon$-DP.
This definition can be generalized to other metrics for capturing dataset proximity depending on the application, e.g., the Manhattan distance metric used to provide indistinguishability if the individual's registration date differs at most 5 days in two databases, and the Euclidean distance on the 2-dimensional space used to preserve the user's longitude and latitude information~\cite{chatzikokolakis2015constructing}. In particular, for text data, we adopt metric Differential Privacy (a.k.a. $d_{\chi}$-privacy), following~\cite{chatzikokolakis2013broadening,fernandes2019generalised,feyisetan2020privacy}. In this framework, we ensure that for all $y\in\mathcal{Y}$, it holds that:
\begin{align}
\text{Pr}[M(x)=y]\leq e^{\epsilon d(x,x')}\text{Pr}[M(x')=y],
\end{align}
where the metric $d(x,x')=\left\|\phi(x)-\phi(x')\right\|$ describes the Euclidean distance of the word representations for $x,x'$ in some semantic embedding space like GloVe~\cite{pennington2014glove}. Under this definition, the likelihood of a similar output from the mechanism is weighted in proportion to distance of the word being substituted.

\paragraph{Multivariate Laplace Mechanism}\label{sec:laplace}
A popular approach for achieving metric-DP is to use a multivariate Laplace Mechanism for high-dimensional data~\cite{wu2017bolt,feyisetan2020privacy}. Given the embedding vector $\phi(x)\in\mathcal{R}^n$ for each word in the vocabulary, an $n$-dimensional noise $\kappa$ is sampled following the distribution $p(\kappa) \propto \exp(-\epsilon\left\|\kappa\right\|)$. This variate is obtained by first sampling a uniform vector in the $n$-dimensional unit ball and scaling it using a Gamma variate sampled from $\Gamma(n,1/\epsilon)$. The perturbed word $x'$ is the nearest word to $\phi(x)+\kappa$ in the embedding space.

\paragraph{Truncated Poisson Sampling} The mechanism we define in this paper uses random variates sampled from a Poisson distribution, but truncated in value if it gets too large. We define this density function below.
\begin{definition}\label{def:poisson}
Let $\lambda > 0$ be a real and $a,b$ be two integers with $1 \leq a < b$. We say that a random variable $X$ follows a $\poisson{\lambda;a,b}$ distribution if the following holds:
\begin{align*}
    \Pr(X = k) &= \begin{cases} 
        \frac{e^{-\lambda}\lambda^k}{k!} &\text{ if }a \leq k < b\\
        1 - \sum_{k=a}^{b-1}\frac{e^{-\lambda}\lambda^k}{k!} &\text{ if }k = b\\
        0 & \text{ otherwise.}
    \end{cases}
\end{align*}
\end{definition}
To sample a random variate $X$ following this distribution, we sample $Y \sim \textsf{Poisson}(\lambda)$ and set $X = Y$ if $a \le Y < b$, and $X = b$, otherwise. An important property of such random variables is that for all $\lambda > 0$, it holds that $\Pr(X = b) > e^{-\lambda}$. This follows from the fact that since $1 \leq a < b$, we can write $\Pr(X = b) = \sum_{k=0}^\infty \frac{e^{-\lambda}\lambda^k}{k!} - \sum_{k=a}^{b-1}\frac{e^{-\lambda}\lambda^k}{k!} = e^{-\lambda} + \sum_{k = 1}^{a-1} \frac{e^{-\lambda}\lambda^k}{k!} + \sum_{k = b+1}^{\infty} \frac{e^{-\lambda}\lambda^k}{k!} > e^{-\lambda}$. This will be useful in our privacy analysis.

\paragraph{Gumbel Distribution} Our mechanism uses random variates sampled from the Gumbel distribution, defined over all $x \in\mathbb{R}$, using the cumulative density $\gmbl(x;\mu, \beta) = \exp\paran{-\exp\paran{-(x-\mu)/\beta}}$ for $\mu \in \mathbb{R}$ and $\beta > 0$. We write $X \sim \gmbl(0, b)$ to denote a Gumbel distributed random sample with $\mu = 0$ and $\beta=b$.

\paragraph{Lambert-W Function} This is a popular multi-valued function obtained from the inverse relation of the function $f(w) = we^w$ for any complex valued $w$. We focus on only the real principal branch of this function defined whenever $f(w) \ge -1$, in which we have the asymptotic identity $W(x) = \ln x - \ln\ln x + \Theta\paran{\frac{\ln\ln x}{\ln x}}$ (see~\cite{hoorfar2008inequalities}). 
% \begin{align}
% \textbf{v}\sim p(x;\mu,\Sigma)&=\frac{\exp\Big(-\frac{1}{2}(x-\mu)^T\Sigma^{-1}(x-\mu)\Big)}{(2\pi)^{n/2}|\Sigma|^{1/2}},\\
% l\sim p(x;n,\theta)&=\frac{x^{n-1}e^{-x/\theta}}{\Gamma(n)\theta^n},
% \end{align}
% where $\theta=1/\epsilon$, $n$ is the embedding dimensionality, the mean $\mu$ is centered at the origin, the covariance matrix $\Sigma$ is the identity matrix. 

% \section{Randomized Perturbations using Laplace Noise}
\section{Overview of our Approach}
We now give an overview of approaches discussed in this paper for defending against adversarial attacks. 
Given text input $x\in\mathcal{X}$, we consider classification tasks where a model $f(x;\theta)$, parametrized by $\theta$, should predict a label $y\in\mathcal{Y}$. For sentiment classification tasks, the input $x$ is composed of a string of $l$ words $x_1,x_2,\cdots,x_l$ and labelled by one of the two classes $y\in\{1,-1\}$, where the positive sentiment is denoted by $1$ while the negative by $-1$. For textual entailment tasks, two texts are given, one is the premise $x$ and the other is the hypothesis $x'$, and a label is provided based on the relationship between the two: $y\in\{0, 1, 2\}$ denoting the entailment, contradiction or neutral relationship, respectively. Performance of the classification model is evaluated by the percentile of correct predictions inferred on the testing set: $\sum_{x_i\in\mathcal{D}_{\text{test}}}\mathbbm{1}(f(x_i;\theta)= y_i)/|\mathcal{D}_{\text{test}}|$, where $\mathbbm{1}$ is an indicator function equal to $1$ if the predicted label $f(x_i;\theta)$ is identical to the ground-truth $y_i$, $0$ otherwise; $|\mathcal{D}_{\text{test}}|$ represents the size of the test set.

\paragraph{Adversarial Attacks by Word Substitutions }
We evaluate the performance of existing certifiably robust trained models when perturbed texts are provided as inputs. Formally, a word-level perturbation is obtained by substituting a given word $x_i$ by another word $\widetilde{x}_i$ in a way that the semantic similarity between the two is determined by the leveraged metric DP mechanism. To achieve this, the additive noise is parametrized by the privacy parameter $\epsilon$: a larger value of $\epsilon$ corresponds to less noise, and vice versa.
%helps to preserve the original word's semantic meaning while a smaller value probably leads to an irrelevant word for substitution, which can hamper the utility of downstream tasks trained on the perturbed dataset.

For the multivariate Laplace Mechanism of~\cite{feyisetan2020privacy}, since the noise is scaled purely as a function of the distance from the original word, when $\epsilon$ is small, words in the dense regions of the embedding space are prone to getting substituted with dissimilar words (that are further away), compared to the words in the sparse region. This is because in areas where embedding vectors are densely located, the distance between two irrelevant words is commensurate to that between two words with similar meanings in a sparse region. Hence, adapting the word-level substitution to variations in the density of the embedding space can help boost the utility of models trained on perturbed datasets. To do this efficiently (and without any expensive computation of local sensitivity each time a substitution is made), we propose a novel mechanism based on a truncated Gumbel distribution and prove that it admits metric DP. Instead of sampling based on the distance from the original word, this approach samples $k$ candidate substitutions following the Truncated Poisson distribution and then makes a distance-based calibrated random choice from the $k-1$-nearest neighbors of the original word in the embedding space (see Algorithms~\ref{alg:gumbelargmin} and~\ref{alg:gumbel_perturbation}). We describe this mechanism in more detail in Section~\ref{sec:gumbel}, and prove its formal privacy guarantees in Appendix~\ref{sec:proofGumbel}.
% : the noise magnitude parameterized by $\epsilon$ is small when a large $\epsilon$ is used but increases when the value of $epsilon$ decreases.
\paragraph{Learning with Adversarial Examples } Motivated by the success of augmented training approaches when text perturbations happen in the form of extraneous text insertion~\cite{jia2017adversarial}, paraphrasing~\cite{iyyer2018adversarial,ribeiro2018semantically},  character-level noise~\cite{belinkov2017synthetic,ebrahimi2017hotflip}, we also investigate the effectiveness of adding adversarial examples generated by metric DP mechanisms to the training set for retraining. Retaining the label of each sample, we perturb the text four times, during which every word is perturbed by either the existing multivariate Laplace Mechanism or the proposed truncated Gumbel Mechanism.
\section{Truncated Gumbel Mechanism}\label{sec:gumbel}
Motivated by the approach proposed by~\cite{durfee2019practical}, our density-aware word substitution mechanism uses a Gumbel random variate for selecting amongst a list of candidate perturbations (see Algorithm~\ref{alg:gumbel_perturbation}). To ensure plausible deniability over the entire vocabulary, the support of the substitution mechanism must include all the words, however, limiting the set of candidate substitutions to only the semantically similar words is necessary to maintain utility. 

We balance this trade-off by first randomly selecting the $k$ nearest neighbors of the original word using a truncated Poisson variate, with support over the whole vocabulary (see Step 4). The mean number of candidates is set to the natural logarithm of the vocabulary size, to ensure that this number is neither too small, nor too large. Next, the closest $k-1$ words to the original word are obtained (using a nearest neighbor search) and their distances are recorded (see Steps 5 and 6). A random choice over this set is made using Algorithm~\ref{alg:gumbelargmin}, where the distances are first noised with Gumbel distributed random variates and then the smallest noised distance determines the new word (see Step 7). The Gumbel noise is scaled using the privacy parameter $\epsilon$ and the diameter $\Delta$ of the embedding space, and then clipped using a truncation parameter $C > 0$. The process is repeated independently for each word in the input string.

\begin{algorithm}[t]\small
\SetKwInOut{Input}{Input}
\Input{Real vector 
$u = [u_1,\dots,u_m]$,
scale parameter $b > 0$, truncation parameter $C > 0$}
Sample $g_1,\dots,g_m \sim_{i.i.d.} \gumbel{0}{b}$ truncated between $[-C, C]$.\\
Compute $u' = [u_1+g_1,\dots,u_m+g_m]$.\\
\Return $\arg\min u'$.
\caption{$\gumbelMech$}
\label{alg:gumbelargmin}
\end{algorithm}

\begin{algorithm*}[t]\small
\SetKwInOut{Input}{Input}
\SetKwInOut{Parameter}{Parameter}
\SetKw{Continue}{continue}
\SetKw{Return}{Return}
\SetAlgoLined
\Input{String $x = w_1 w_2 \dots w_{\ell} \in \mathcal{W}^{\ell}$, privacy parameter $\epsilon>0$, word set $\mathcal{W}$.}
Let $\Delta=\max_{w,w'\in \mathcal{W}}\|\phi(w)-\phi(w')\|_2$ be the maximum inter-word distance, $\Delta_0=\min_{\substack{w,w'\in \mathcal{W}\\w\neq w'}}\|\phi(w)-\phi(w')\|_2$ be the minimum inter-word distance. Set $b = \frac{2\Delta}{\min\{W(2\alpha \Delta),\ \log_e(\alpha \Delta_0)\}}$, where $\alpha = \frac{1}{3}\paran{\epsilon - \frac{2(1+\log |\mathcal{W}|)}{\Delta_0}}$ and $W$ denotes the principal branch of the Lambert-W function.\\
Initialize an empty string $\tilde{x}$.\\
\For{$w_i \in x$}{
Sample $k = \poisson{\log |\mathcal{W}|; 1,|\mathcal{W}|}$.\\
Find the top $k$ closest words to $w_i$ in $\mathcal{W}$ as  $\mathbf{u} = [u_1,u_2,\dots,u_j,\dots,u_k]$, where $u_1 = w_i$.\\
Compute the distances $\mathbf{d} = [d_1,d_2,\dots d_j,\dots, d_k]$, where $d_j = ||w_i - u_j||_2$.\\
Set $\widetilde{w}_i = u_j$, where $j = \gumbelArgMin{\mathbf{d}}{b}{\Delta}$.\\
Add $\widetilde{w}_i$ to $\tilde{x}$.}
 \Return{$\tilde{x}$.}
 \caption{Truncated Gumbel Perturbation Mechanism}
 \label{alg:gumbel_perturbation}
\end{algorithm*}

\section{Experimental Results}
We evaluate the proposed privacy mechanism, adversarial attacks and the defense approach through answers to the following questions:
% \begin{itemize}
% \item \textbf{Q1}: 
% \end{itemize}
\begin{enumerate}[label=\textbf{Q}\textbf{{\arabic*}}]
\item How does the privacy parameter $\epsilon$ affect the behavior of the perturbation mechanisms on different text classification tasks?
\item Does the proposed truncated Gumbel mechanism lead to a smaller range of word substitutions compared to the Multivariate Laplace Mechanism?%, so as to preserve word semantics and further improve utility?
\item How will different adversarial training approaches, i.e., the IBP approach with certified robustness and the proposed augmented training, perform when testing on adversarial examples derived from metric-DP mechanisms?
% \item Will the approach with certified robustness maintain the good performance when attacked by $d_{\chi}$-privacy mechanisms?
% \item What is the accuracy under attack achieved by the proposed augmented training and how does it compare to that of certified defenses?
\end{enumerate}
\subsection{Tasks and Datasets}
We evaluate the robustness of models on two text classification tasks: sentiment analysis on the IMDb movie review dataset~\cite{maas2011learning} and textual entailment on premise-hypothesis relation dataset SNLI~\cite{bowman2015large}. We use 300-dimensional GloVe vectors for word embedding~\cite{pennington2014glove}. The statistics of the two datasets are listed in Table.~\ref{tab:dataset_summary}.
\paragraph{Sentiment Analysis} In IMDb, each movie review is accompanied with either a positive or negative label. For the binary classification task, we implemented the CNN architecture that achieved the best adversarial attack and certified accuracy in~\cite{jia2019certified}.%: word vectors are first fed into a one-layer convolutional neural network ($\text{kernel size}=3, \text{stride}=1$) followed by an average operation on the resulted hidden state; the vector representing the text is further passed to a 2-layer feedforward network ($\text{hidden size}=100$) to compute the probability of labelling this text as positive (or negative).
\paragraph{Textual Entailment} In SNLI, each sample is composed of two sentence: one as the premise and the other as the hypothesis. The classification task is to define the relationship as an entailment, contradiction, or neutral. Following the implementation in~\cite{alzantot2018generating}, only words in hypothesis are allowed to be substituted. Similarly, we adopted the architecture that outperformed others in~\cite{jia2019certified} for evaluating different adversarial training approaches.%: the premise and the hypothesis are encoded separately by summing the word vectors; the vectors representing premise and the hypothesis are concatenated and further fed into a 3-layer feedforward network ($\text{hidden size}=200$) to obtain the final logit.

\begin{table}[t!]
\centering
  \begin{tabular}{lll}
  \toprule
  \textbf{Dataset}&\textbf{IMDb}&\textbf{SNLI}\\\midrule
  Task type&binary&three-class\\
  Training set size&20,000&550,152\\
  Testing set size&1000&10,000\\
  Total word count&11,856,015&4,614,822\\
  Vocabulary size&145,901&49,895\\
  Sentence length&263.46$\pm$195.29&8.25$\pm$3.20\\\bottomrule
  \end{tabular}
  \caption{Summary of dataset properties.}
  \label{tab:dataset_summary}
\end{table}

\subsection{Compared Approaches}
We compare robustness of the following two training approaches when adversarial examples are generated using metric-DP perturbation.

\paragraph{Certifiably Robust Trained Approach} Interval Bound Propagation (IBP) was leveraged to minimize the upper bound on the worst-case loss that any combination of word substitutions can induce. Specifically, an upper and lower bound on the activation of an neuron in each layer is computed based on the bounds of neurons in previous layers that connect to it. Bounds for the input layer is computed based on the smallest axis-aligned box that contains all the possible word substitutions, while the upper bound on the loss in the final layer is combined with the normal cross entropy loss to optimize the classification performance on the actual word and any other substitutions. The allowed substitutions are based on~\cite{alzantot2018generating}.

\paragraph{Augmented Training} we add the adversarial examples (four times of perturbations per sample) generated by metric differential privacy mechanisms into the training set and retain the model.

\subsection{Adversarial Attack Methodology~\label{sec:attack_methodology}}
%\zekun{The 'Attack' in this paper refers to affecting the utility of machine learning models by introducing adversarial examples. We may need to emphasize or reiterate in different places so that the reader won't confuse it with membership inference attack because we are writing a lot about privacy mechanisms.}
Following~\cite{alzantot2018generating}, a population-based genetic attacker is implemented to search for perturbations that lead to misclassification from the model. Given an original or modified sentence, the attacker randomly substitutes a word from the sentence with a new one based on the perturbation mechanism satisfying metric DP. After multiple substitutions, the attacker obtains a population of new sentences together with their fitness scores (negatively proportional to the probability predicted for the correct label). 

If the new sentence with the highest fitness score successfully fools the model, then the attacker moves forward to the next sentence and starts a new round of testing. Otherwise, the attacker will perform crossover and mutation operations: sample two new sentences as parents from the population according to their fitness score, and then generate the child sentence by taking the word from either parent randomly. Another round of perturbation over the child sentence is then performed to further increase sentence diversity. The model is certified robust to after providing correct predictions over a predefined numbers of attacks.

\subsection{Evaluation Metrics}
Based on attributes of the testing set, different metrics are utilized to evaluate models' performance. 
\begin{itemize}
\item Clean Accuracy: the percentage of correct predictions when testing on the original samples.
\item Adversarial Accuracy: percent of correct predictions when testing on perturbed samples.
\end{itemize}
\subsection{Privacy Statistics of Metric DP Mechanisms}
In the context of privacy preservation, plausible deniability measures the likelihood of making correct inference given a sample perturbed by the privacy mechanism. %We perturb words from vocabularies of the two datasets with the metric DP mechanisms at different values of $\epsilon$ for $1,000$ times.
Following~\cite{feyisetan2020privacy}, the following statistics are recorded to empirically evaluate the plausible deniability of the metric DP mechanisms at different values of $\epsilon$ (over $1,000$ experiment runs):
\begin{itemize}
\item $N_w$, measures the probability that a word does not get modified by the mechanism. This is approximated by counting the number of times an input word $w$ does not get replaced after running the mechanism $1,000$ times.
\item $S_w$, which is the number of distinct words that are produced as the output of $M(w)$. This is approximated by counting the number of distinct substitutions for an input word $w$ after running the mechanism $1,000$ times.
\end{itemize}

\begin{figure*}[t!]
\centering
% \hspace{-.4in}
\subfloat[Multivariate Laplace Mechanism on IMDb]{\label{fig:laplace_IMDB} \includegraphics[width=0.49\textwidth]{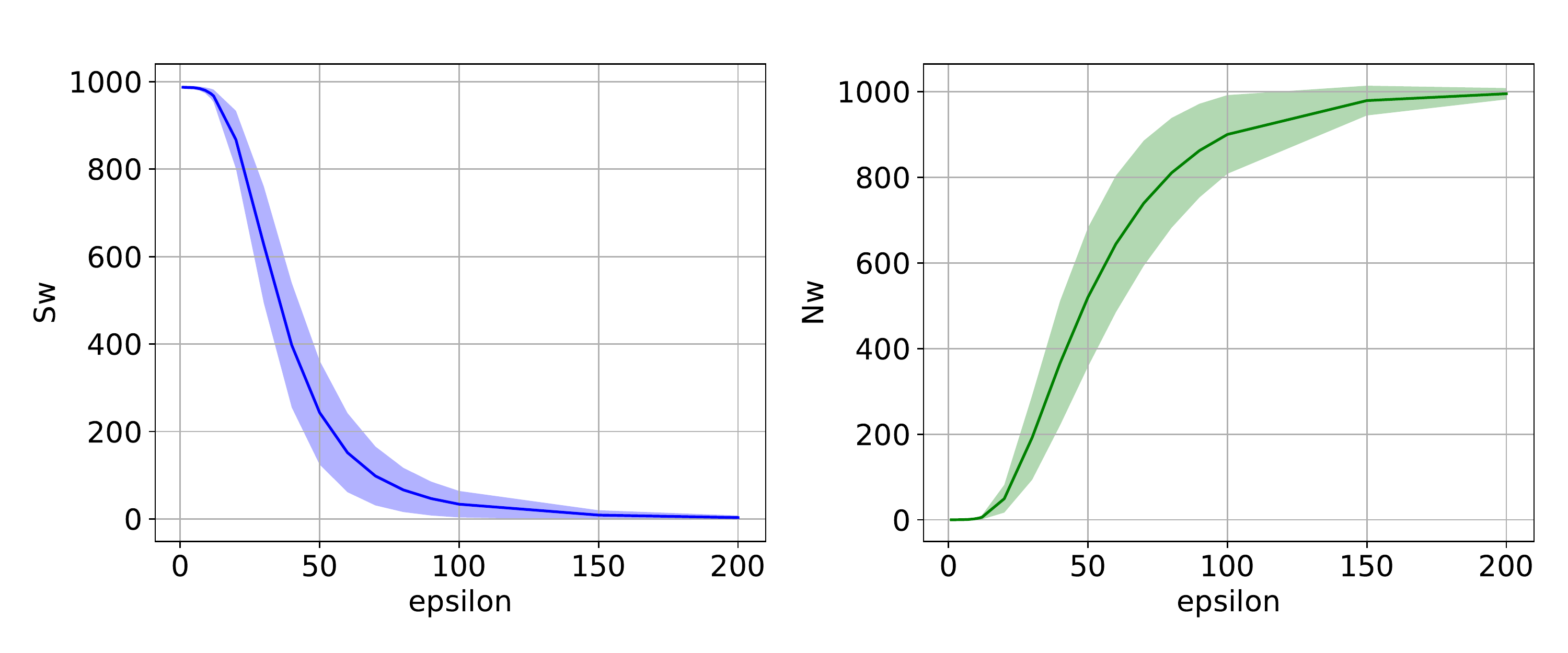}}
\subfloat[Multivariate Laplace Mechanism on SNLI]{\label{fig:laplace_SNLI} \includegraphics[width=0.49 \textwidth]{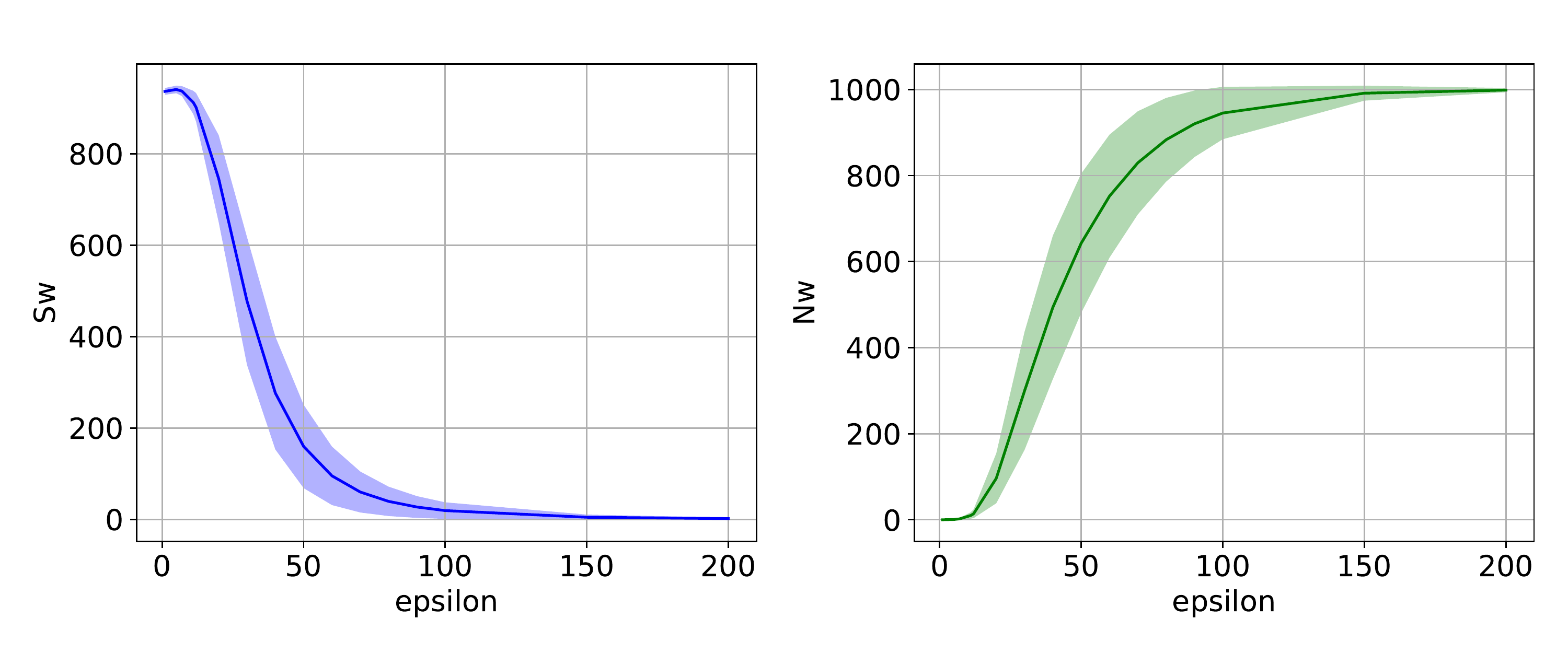}}\\
\subfloat[Truncated Gumbel Perturbation Mechanism on IMDb]{\label{fig:gumbel_IMDB} \includegraphics[width=0.49\textwidth]{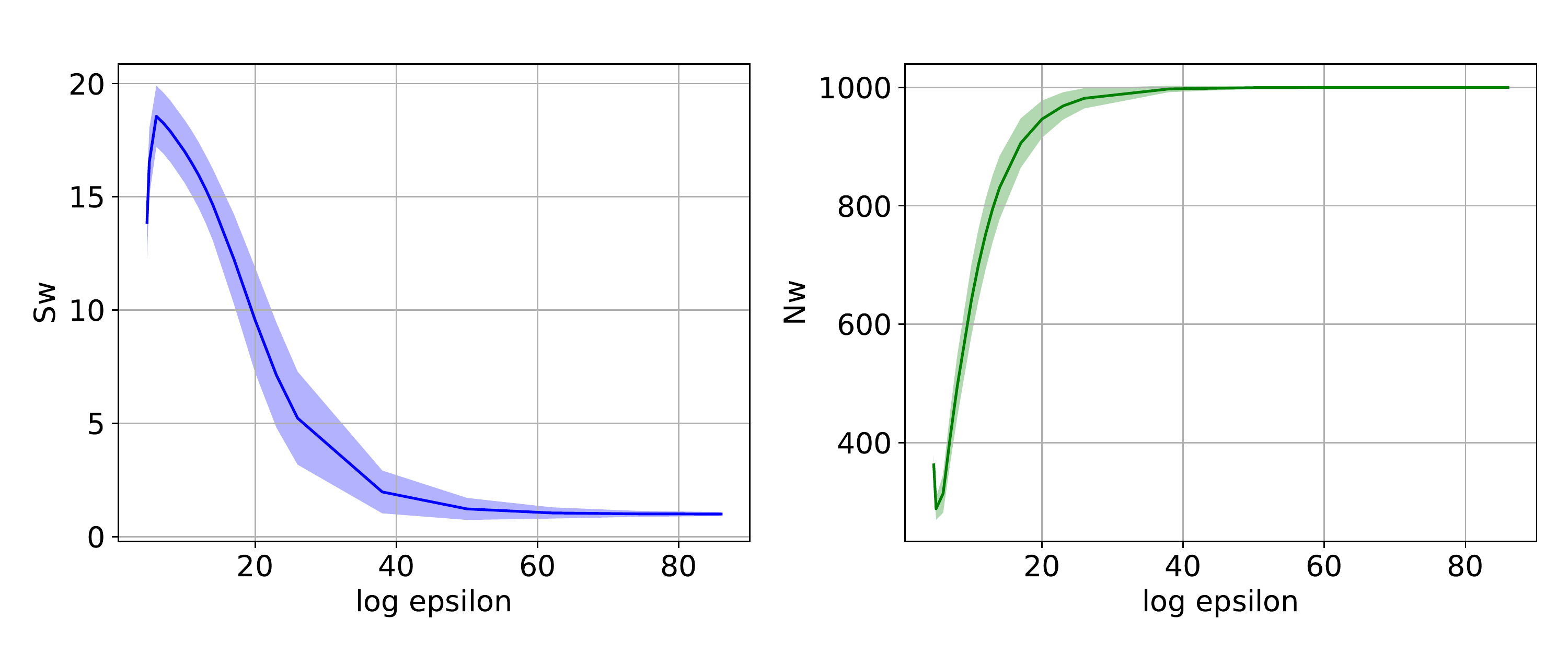}}
\subfloat[Truncated Gumbel Perturbation Mechanism on SNLI]{\label{fig:gumbel_SNLI} \includegraphics[width=0.49 \textwidth]{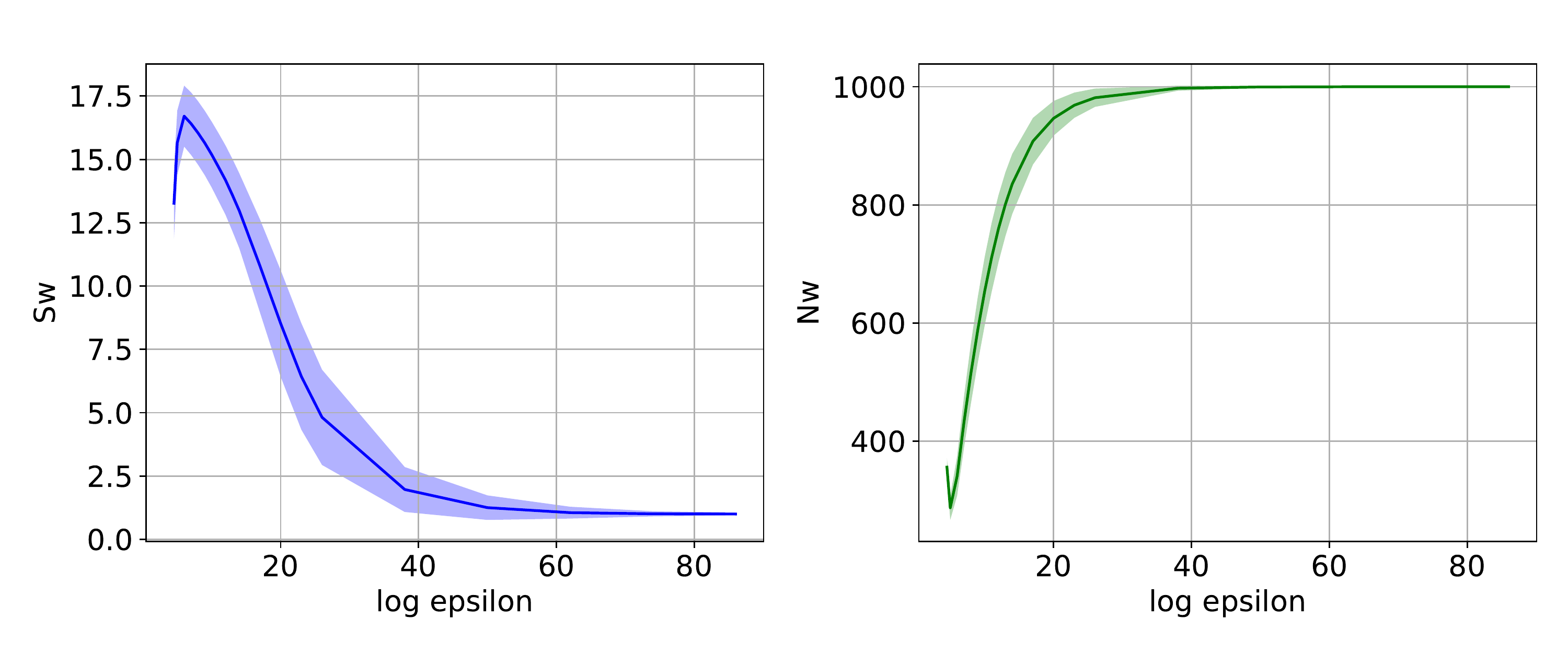}}
\caption{Empirical $S_w$ and $N_w$ statistics of Multivariate Laplace Mechanism and Truncated Gumbel Perturbation Mechanism on vocabularies from IMDb and SNLI. The average amount of the two measures is plotted as curves while the standard deviation is represented by shadows along the curve. Same plot patterns (curve and shadow) represent the same meaning ($\text{mean}\pm\text{std}$) in the following figures.}
\label{fig:plausible_deniability}
\end{figure*}

\paragraph{Plausible Deniability Analysis (Q1)} 
In Fig.~\ref{fig:plausible_deniability}, we observe similar trends on the two privacy statistic measures for both datasets.
When samples are perturbed by the multivariate Laplace mechanism (shown in Fig.~\ref{fig:laplace_IMDB} and Fig.~\ref{fig:laplace_SNLI}), the number of distinct substitutions $S_w$ decreases from $1,000$ to $0$ while the the times of maintaining the original word $N_w$ shows the opposite trend. The empirical values of the two measures are consistent with the definition of metric DP that the multivariate Laplace mechanisms satisfies i.e.,: $\epsilon\rightarrow0$ provides absolute privacy as the output produced by the mechanism becomes independent of the input word, while $\epsilon\rightarrow\infty$ results in null privacy where $M(w)=w$. 

There are two main differences between truncated Gumbel (demonstrated in Fig.~\ref{fig:gumbel_IMDB} and Fig.~\ref{fig:gumbel_SNLI}) and multivariate Laplace mechanism in privacy statistics: 1) minor increase or decrease in $\epsilon$ does not influence word substitutions produced by truncated Gumbel, hence variation of $S_w$ and $N_w$ is plotted against the logarithm value of $\epsilon$; 2) due to the effects of word substitutions among the top $k$ closest words in the vocabulary,  the maximum amount of distinct substitutions one word can have is around $20$ on IMDB and $17.5$ on SNLI.

\paragraph{Word Substitution Range Analysis (Q2)}
One main advantage of the proposed truncated Gumbel perturbation mechanism over the existing multivariate Laplace mechanism relies on the top-k closest words as substitutions, which helps preserve word semantics and improve utility of downstream ML tasks for words located in dense area of the embedding space. To show this property, we compare the amount of distinct word substitutions $S_w$ when the times of keeping the word unchanged $N_w$ is fixed in Fig~\ref{fig:Sw_against_Nw}. We discover that when different mechanisms result in the same perturbation effects, the multivariate Laplace mechanism has a much broader range of word substitutions compared with the proposed truncated Gumbel mechanism, which will probably raise problems in semantic preservation and result in poor performance on downstream tasks trained on the perturbed dataset.

\begin{figure}[h]
\centering
% \hspace{-.4in}
\subfloat[$S_w$ against $N_w$ on IMDB]{\label{fig:Sw_Nw_IMDB} \includegraphics[width=0.24\textwidth]{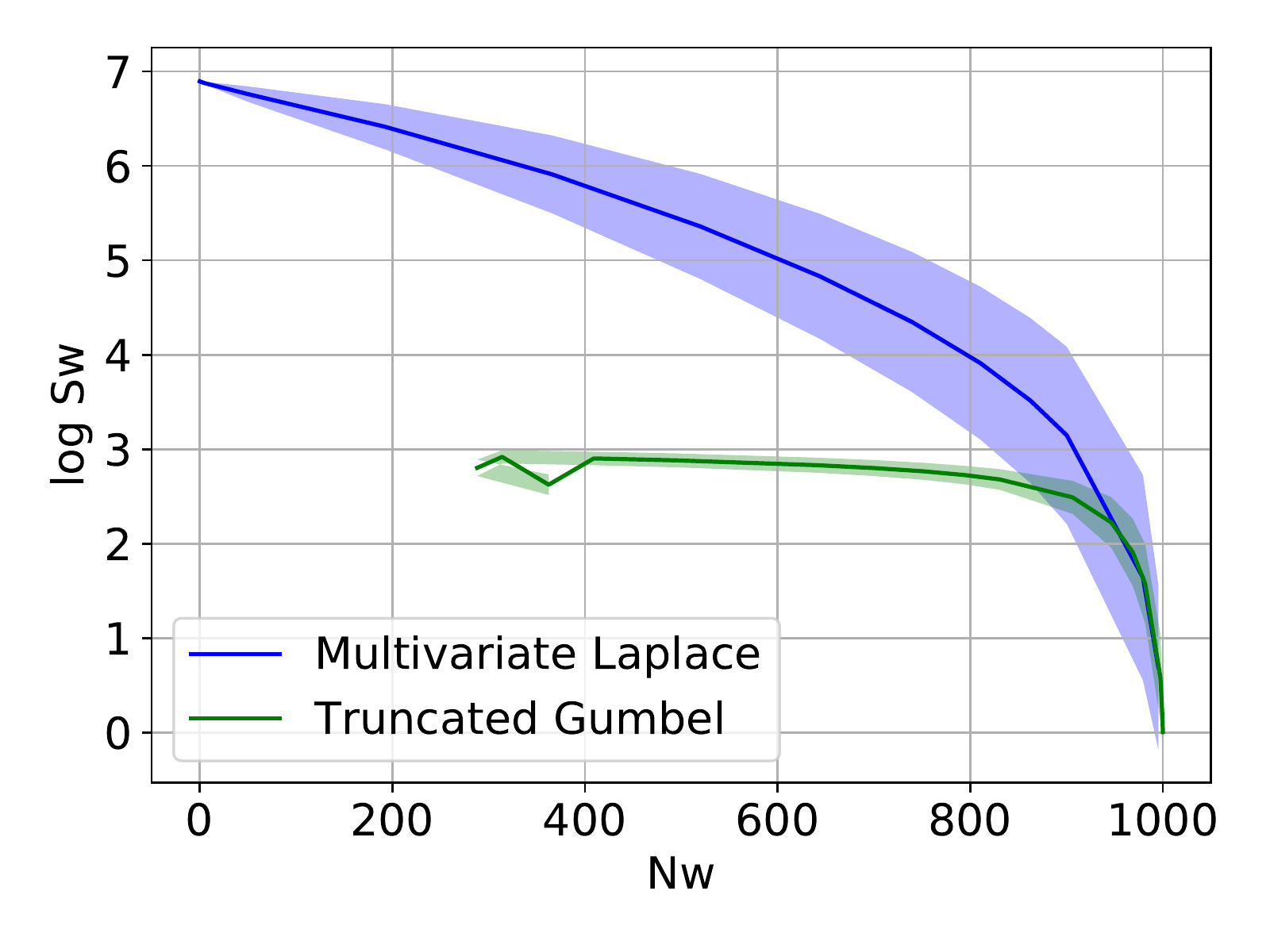}}
\subfloat[$S_w$ against $N_w$ on SNLI]{\label{fig:Sw_Nw_SNLI} 
\includegraphics[width=0.24 \textwidth]{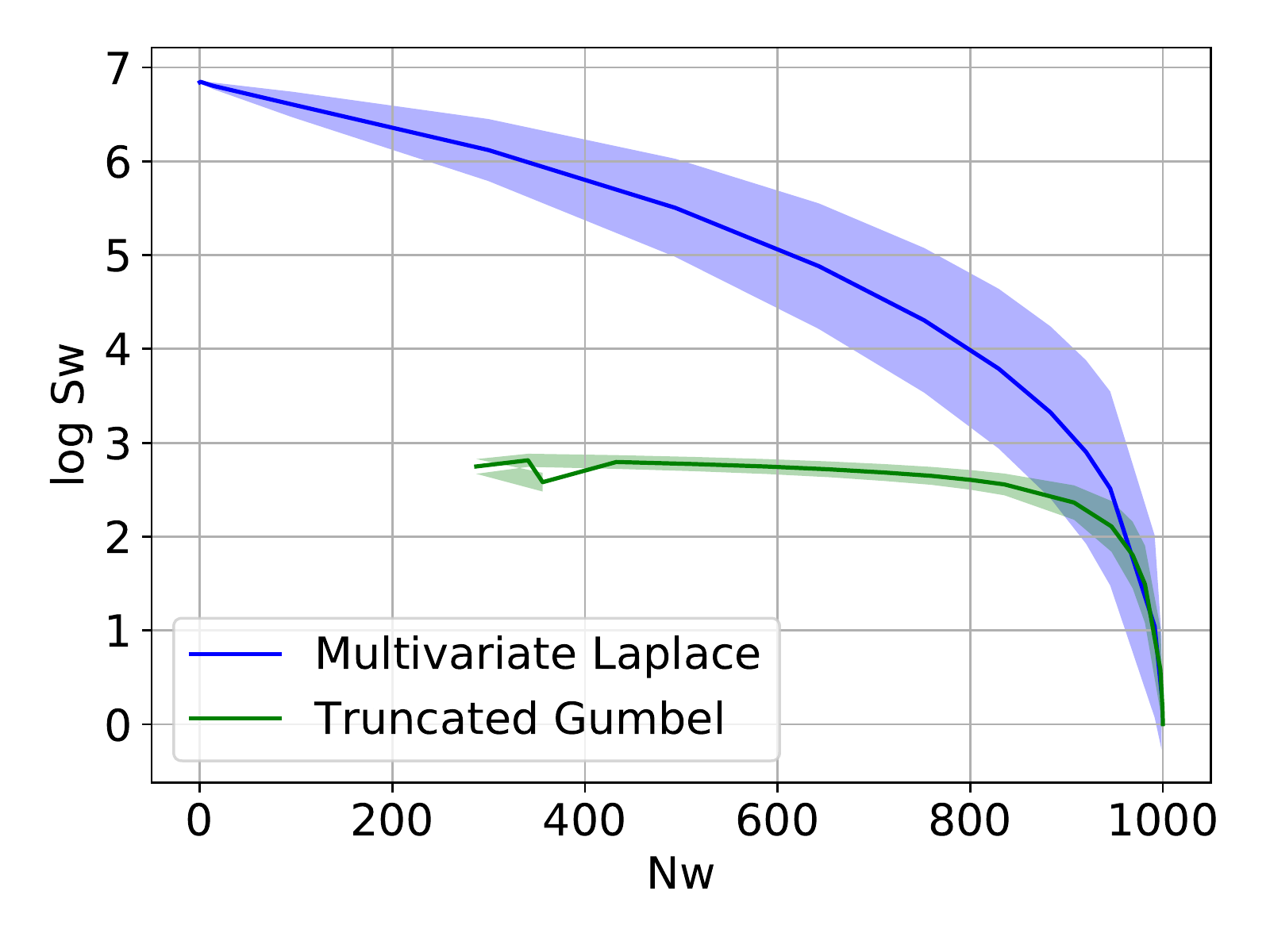}}
\caption{Word substitution range comparison (lower $S_w$ is better when $N_w$ is fixed). Due to the different scales of $S_w$ by the two mechanisms, the y-axis indicates the log value of $S_w$ for better visualization.}
\label{fig:Sw_against_Nw}
\end{figure}

\begin{table*}[t!]\small
\caption{Performance of adversarial training approaches on Text Data with(out) perturbations from \textbf{truncated gumbel perturbation mechanism}. Note that results are recorded when $\log\epsilon=4.67$ for IMDb and $\log\epsilon=4.52$ for SNLI, which are slightly larger than their respective lower bounds on $\epsilon$.}
\label{tab:gumbel_results}
\centering
\begin{tabular}{@{}cllllllllllll@{}}
\toprule
\multicolumn{3}{c}{$\log\epsilon$}                                       & \begin{tabular}[c]{@{}l@{}}4.67/\\ 4.52\end{tabular} & 10.00          & 14.00          & 17.00          & 23.00          & 38.00          & 50.00          & 62.00          & 74.00          & 86.00          \\ \midrule
\multirow{4}{*}{IMDb}                     & \multirow{2}{*}{Clean} & IBP & 81.00                                                & 81.00          & 81.00          & 81.00          & 81.00          & 81.00          & 81.00          & 81.00          & 81.00          & 81.00          \\\cmidrule(lr){3-13}
                                          &                        & Aug & \textbf{89.80}                                       & \textbf{89.60} & \textbf{88.10} & \textbf{90.00} & \textbf{88.30} & \textbf{89.20} & \textbf{89.00} & \textbf{89.40} & \textbf{89.80} & \textbf{89.70} \\\cmidrule(lr){2-13}
                                          & \multirow{2}{*}{Adv}   & IBP & \textbf{35.30}                                                & 34.60          & \textbf{47.40} & 58.60          & 70.90          & 79.90          & 80.80          & 80.90          & 80.90          & 81.00          \\\cmidrule(lr){3-13}
                                          &                        & Aug &32.00                                       & \textbf{34.90} & 43.30          & \textbf{60.20} & \textbf{71.80} & \textbf{86.20} & \textbf{88.80} & \textbf{89.30} & \textbf{89.70} & \textbf{89.70} \\ \midrule
\multicolumn{1}{l}{\multirow{4}{*}{SNLI}} & \multirow{2}{*}{Clean} & IBP & \textbf{79.19}                                       & 79.19          & 79.19          & 79.19          & 79.19          & 79.19          & 79.19          & 79.19          & 79.19          & 79.19          \\\cmidrule(lr){3-13}
\multicolumn{1}{l}{}                      &                        & Aug & 78.89                                                & \textbf{79.92} & \textbf{81.32} & \textbf{81.74} & \textbf{81.77} & \textbf{82.20} & \textbf{82.18} & \textbf{81.86} & \textbf{81.65} & \textbf{81.96} \\\cmidrule(lr){2-13}
\multicolumn{1}{l}{}                      & \multirow{2}{*}{Adv}   & IBP & 12.5                                                 & 11.49          & 12.98          & 14.95          & \textbf{24.01} & \textbf{58.78} & 74.51          & 78.18          & 78.88          & 79.12          \\\cmidrule(lr){3-13}
\multicolumn{1}{l}{}                      &                        & Aug & \textbf{21.05}                                       & \textbf{17.34} & \textbf{16.57} & \textbf{17.05} & 23.96          & 58.58          & \textbf{76.54} & \textbf{80.62} & \textbf{81.41} & \textbf{81.90} \\ \bottomrule
\end{tabular}
\end{table*}

\subsection{Model Robustness Against Metric DP Adversarial Samples (Q3)}
We list performance of the two adversarial training approaches when samples are perturbed by the multivariate Laplace mechanism in Table~\ref{tab:laplace_results} and the truncated Gumbel mechanism in Table~\ref{tab:gumbel_results}. 
% (\nan{Maybe we should put laplace results in appendix to save some space})
 
 In Table~\ref{tab:laplace_results}, clean accuracy of the proposed augmented training approach is approximately $8.74\%$ higher than that of the certifiably robust trained approach IBP for any $\epsilon$ selection on IMDb and $3.33\%$ higher for $\epsilon\ge40$ on SNLI. Retraining with adversarial examples helps maintain the similar level of clean accuracy as the normal training approach, which is consistent with observations in literature~\cite{jia2017adversarial,iyyer2018adversarial,ribeiro2018semantically,belinkov2017synthetic,ebrahimi2017hotflip}. 
 When evaluating the model's robustness against word perturbations from the multivariate Laplace mechanism, the augmented training outperforms the IBP approach only when the $\epsilon$ value is larger than some threshold, e.g., $\epsilon>150$ on IMDb and $\epsilon>60$ on SNLI. This is expected as the augmented training cannot protect against all attacks especially when small values of $\epsilon$ results in any word substitution without considering semantic-preserving. In this case, the model can hardly learn the hidden relationship between the corrupted new texts and the original text label.

Given better semantic-preserving capability inherent in the proposed truncated Gumbel mechanism, the augmented training approach outperforms the certifiably robust trained IBP method in both clean and adversarial accuracy almost for any tested $\epsilon$ value tested. In Table~\ref{tab:gumbel_results}, improvement of clean accuracy by the augmented training approach over IBP is $9.87\%$ on IMDb and $3.77\%$ on SNLI when $\log \epsilon = 50$. At the same time, better performance against adversarial attacks is achieved by the augmented training approach: $9.90\%$ higher adversarial accuracy on IMDb and $2.72\%$ on SNLI.

One possible explanation of the inferior adversarial accuracy achieved by the certified defense approach IBP may be attributed to the training procedure, which is based on the word substitutions that preserve semantic meanings~\cite{alzantot2018generating}. However, the testing adversarial examples are generated by randomized perturbations from metric DP mechanisms, where the semantic meaning is not always preserved, but dynamically determined by the privacy parameter $\epsilon$. 

\section{Discussion and Conclusion}
We study the performance of different adversarial training approaches against adversarial examples generated by metric DP mechanisms. To better preserve semantic meanings during word perturbations, we propose a novel truncated Gumbel mechanism, which formally satisfies metric DP (see Appendix~\ref{sec:proofGumbel}). Empirically experiments demonstrate the advantage of the truncated Gumbel mechanism over the existing multivariate Laplace mechanism due to its smaller range of substitution candidates. In two textual classification tasks, retraining with adversarial examples performs better than the certified defence in both clean and adversarial accuracy.

We think the following aspects are interesting and deserve more investigations in the future: 1) robustness of other adversarial training approaches based on the metric DP-inspired adversarial examples, e.g., surrogate-loss minimization; 2) generalization capability of the well-trained augmented training approach, e.g., performance against other types of adversarial examples; 3) privacy preservation performance of the proposed truncated gumbel mechanism, e.g., performance of membership inference attacks (MIA) on perturbed texts.

\newpage
\balance
\bibliographystyle{acl_natbib}
\bibliography{emnlp2020}

\begin{thebibliography}{34}
\expandafter\ifx\csname natexlab\endcsname\relax\def\natexlab#1{#1}\fi

\bibitem[{Alzantot et~al.(2018)Alzantot, Sharma, Elgohary, Ho, Srivastava, and
  Chang}]{alzantot2018generating}
Moustafa Alzantot, Yash Sharma, Ahmed Elgohary, Bo-Jhang Ho, Mani Srivastava,
  and Kai-Wei Chang. 2018.
\newblock Generating natural language adversarial examples.
\newblock \emph{arXiv preprint arXiv:1804.07998}.

\bibitem[{Andr{\'e}s et~al.(2013)Andr{\'e}s, Bordenabe, Chatzikokolakis, and
  Palamidessi}]{andres2013geo}
Miguel~E Andr{\'e}s, Nicol{\'a}s~E Bordenabe, Konstantinos Chatzikokolakis, and
  Catuscia Palamidessi. 2013.
\newblock Geo-indistinguishability: Differential privacy for location-based
  systems.
\newblock In \emph{Proceedings of the 2013 ACM SIGSAC conference on Computer \&
  communications security}, pages 901--914.

\bibitem[{Belinkov and Bisk(2017)}]{belinkov2017synthetic}
Yonatan Belinkov and Yonatan Bisk. 2017.
\newblock Synthetic and natural noise both break neural machine translation.
\newblock \emph{arXiv preprint arXiv:1711.02173}.

\bibitem[{Bojanowski et~al.(2017)Bojanowski, Grave, Joulin, and
  Mikolov}]{bojanowski2017enriching}
Piotr Bojanowski, Edouard Grave, Armand Joulin, and Tomas Mikolov. 2017.
\newblock Enriching word vectors with subword information.
\newblock \emph{Transactions of the Association for Computational Linguistics},
  5:135--146.

\bibitem[{Bojarski et~al.(2016)Bojarski, Del~Testa, Dworakowski, Firner, Flepp,
  Goyal, Jackel, Monfort, Muller, Zhang et~al.}]{bojarski2016end}
Mariusz Bojarski, Davide Del~Testa, Daniel Dworakowski, Bernhard Firner, Beat
  Flepp, Prasoon Goyal, Lawrence~D Jackel, Mathew Monfort, Urs Muller, Jiakai
  Zhang, et~al. 2016.
\newblock End to end learning for self-driving cars.
\newblock \emph{arXiv preprint arXiv:1604.07316}.

\bibitem[{Bowman et~al.(2015)Bowman, Angeli, Potts, and
  Manning}]{bowman2015large}
Samuel~R Bowman, Gabor Angeli, Christopher Potts, and Christopher~D Manning.
  2015.
\newblock A large annotated corpus for learning natural language inference.
\newblock \emph{arXiv preprint arXiv:1508.05326}.

\bibitem[{Chatzikokolakis et~al.(2013)Chatzikokolakis, Andr{\'e}s, Bordenabe,
  and Palamidessi}]{chatzikokolakis2013broadening}
Konstantinos Chatzikokolakis, Miguel~E Andr{\'e}s, Nicol{\'a}s~Emilio
  Bordenabe, and Catuscia Palamidessi. 2013.
\newblock Broadening the scope of differential privacy using metrics.
\newblock In \emph{International Symposium on Privacy Enhancing Technologies
  Symposium}, pages 82--102. Springer.

\bibitem[{Chatzikokolakis et~al.(2015)Chatzikokolakis, Palamidessi, and
  Stronati}]{chatzikokolakis2015constructing}
Konstantinos Chatzikokolakis, Catuscia Palamidessi, and Marco Stronati. 2015.
\newblock Constructing elastic distinguishability metrics for location privacy.
\newblock \emph{Proceedings on Privacy Enhancing Technologies},
  2015(2):156--170.

\bibitem[{Durfee and Rogers(2019)}]{durfee2019practical}
David Durfee and Ryan~M Rogers. 2019.
\newblock Practical differentially private top-k selection with
  pay-what-you-get composition.
\newblock In \emph{Advances in Neural Information Processing Systems}, pages
  3532--3542.

\bibitem[{Dvijotham et~al.(2018)Dvijotham, Gowal, Stanforth, Arandjelovic,
  O'Donoghue, Uesato, and Kohli}]{dvijotham2018training}
Krishnamurthy Dvijotham, Sven Gowal, Robert Stanforth, Relja Arandjelovic,
  Brendan O'Donoghue, Jonathan Uesato, and Pushmeet Kohli. 2018.
\newblock Training verified learners with learned verifiers.
\newblock \emph{arXiv preprint arXiv:1805.10265}.

\bibitem[{Dwork et~al.(2006)Dwork, McSherry, Nissim, and
  Smith}]{dwork2006calibrating}
Cynthia Dwork, Frank McSherry, Kobbi Nissim, and Adam Smith. 2006.
\newblock Calibrating noise to sensitivity in private data analysis.
\newblock In \emph{Theory of cryptography conference}, pages 265--284.
  Springer.

\bibitem[{Ebrahimi et~al.(2017)Ebrahimi, Rao, Lowd, and
  Dou}]{ebrahimi2017hotflip}
Javid Ebrahimi, Anyi Rao, Daniel Lowd, and Dejing Dou. 2017.
\newblock Hotflip: White-box adversarial examples for text classification.
\newblock \emph{arXiv preprint arXiv:1712.06751}.

\bibitem[{Fernandes et~al.(2019)Fernandes, Dras, and
  McIver}]{fernandes2019generalised}
Natasha Fernandes, Mark Dras, and Annabelle McIver. 2019.
\newblock Generalised differential privacy for text document processing.
\newblock In \emph{International Conference on Principles of Security and
  Trust}, pages 123--148. Springer, Cham.

\bibitem[{Feyisetan et~al.(2020)Feyisetan, Balle, Drake, and
  Diethe}]{feyisetan2020privacy}
Oluwaseyi Feyisetan, Borja Balle, Thomas Drake, and Tom Diethe. 2020.
\newblock Privacy-and utility-preserving textual analysis via calibrated
  multivariate perturbations.
\newblock In \emph{Proceedings of the 13th International Conference on Web
  Search and Data Mining}, pages 178--186.

\bibitem[{Feyisetan et~al.(2019)Feyisetan, Diethe, and
  Drake}]{feyisetan2019leveraging}
Oluwaseyi Feyisetan, Tom Diethe, and Thomas Drake. 2019.
\newblock Leveraging hierarchical representations for preserving privacy and
  utility in text.
\newblock \emph{arXiv preprint arXiv:1910.08917}.

\bibitem[{Goodfellow et~al.(2014)Goodfellow, Shlens, and
  Szegedy}]{goodfellow2014explaining}
Ian~J Goodfellow, Jonathon Shlens, and Christian Szegedy. 2014.
\newblock Explaining and harnessing adversarial examples.
\newblock \emph{arXiv preprint arXiv:1412.6572}.

\bibitem[{Gowal et~al.(2018)Gowal, Dvijotham, Stanforth, Bunel, Qin, Uesato,
  Arandjelovic, Mann, and Kohli}]{gowal2018effectiveness}
Sven Gowal, Krishnamurthy Dvijotham, Robert Stanforth, Rudy Bunel, Chongli Qin,
  Jonathan Uesato, Relja Arandjelovic, Timothy Mann, and Pushmeet Kohli. 2018.
\newblock On the effectiveness of interval bound propagation for training
  verifiably robust models.
\newblock \emph{arXiv preprint arXiv:1810.12715}.

\bibitem[{Hoorfar and Hassani(2008)}]{hoorfar2008inequalities}
Abdolhossein Hoorfar and Mehdi Hassani. 2008.
\newblock Inequalities on the lambert w function and hyperpower function.
\newblock \emph{J. Inequal. Pure and Appl. Math}, 9(2):5--9.

\bibitem[{Iyyer et~al.(2018)Iyyer, Wieting, Gimpel, and
  Zettlemoyer}]{iyyer2018adversarial}
Mohit Iyyer, John Wieting, Kevin Gimpel, and Luke Zettlemoyer. 2018.
\newblock Adversarial example generation with syntactically controlled
  paraphrase networks.
\newblock \emph{arXiv preprint arXiv:1804.06059}.

\bibitem[{Jia and Liang(2017)}]{jia2017adversarial}
Robin Jia and Percy Liang. 2017.
\newblock Adversarial examples for evaluating reading comprehension systems.
\newblock \emph{arXiv preprint arXiv:1707.07328}.

\bibitem[{Jia et~al.(2019)Jia, Raghunathan, G{\"o}ksel, and
  Liang}]{jia2019certified}
Robin Jia, Aditi Raghunathan, Kerem G{\"o}ksel, and Percy Liang. 2019.
\newblock Certified robustness to adversarial word substitutions.
\newblock \emph{arXiv preprint arXiv:1909.00986}.

\bibitem[{Kober et~al.(2013)Kober, Bagnell, and
  Peters}]{kober2013reinforcement}
Jens Kober, J~Andrew Bagnell, and Jan Peters. 2013.
\newblock Reinforcement learning in robotics: A survey.
\newblock \emph{The International Journal of Robotics Research},
  32(11):1238--1274.

\bibitem[{Krizhevsky et~al.(2012)Krizhevsky, Sutskever, and
  Hinton}]{krizhevsky2012imagenet}
Alex Krizhevsky, Ilya Sutskever, and Geoffrey~E Hinton. 2012.
\newblock Imagenet classification with deep convolutional neural networks.
\newblock In \emph{Advances in neural information processing systems}, pages
  1097--1105.

\bibitem[{Lecuyer et~al.(2019)Lecuyer, Atlidakis, Geambasu, Hsu, and
  Jana}]{lecuyer2019certified}
Mathias Lecuyer, Vaggelis Atlidakis, Roxana Geambasu, Daniel Hsu, and Suman
  Jana. 2019.
\newblock Certified robustness to adversarial examples with differential
  privacy.
\newblock In \emph{2019 IEEE Symposium on Security and Privacy (SP)}, pages
  656--672. IEEE.

\bibitem[{Maas et~al.(2011)Maas, Daly, Pham, Huang, Ng, and
  Potts}]{maas2011learning}
Andrew Maas, Raymond~E Daly, Peter~T Pham, Dan Huang, Andrew~Y Ng, and
  Christopher Potts. 2011.
\newblock Learning word vectors for sentiment analysis.
\newblock In \emph{Proceedings of the 49th annual meeting of the association
  for computational linguistics: Human language technologies}, pages 142--150.

\bibitem[{Matyasko and Chau(2017)}]{matyasko2017margin}
Alexander Matyasko and Lap-Pui Chau. 2017.
\newblock Margin maximization for robust classification using deep learning.
\newblock In \emph{2017 International Joint Conference on Neural Networks
  (IJCNN)}, pages 300--307. IEEE.

\bibitem[{Mikolov et~al.(2013)Mikolov, Sutskever, Chen, Corrado, and
  Dean}]{mikolov2013distributed}
Tomas Mikolov, Ilya Sutskever, Kai Chen, Greg~S Corrado, and Jeff Dean. 2013.
\newblock Distributed representations of words and phrases and their
  compositionality.
\newblock In \emph{Advances in neural information processing systems}, pages
  3111--3119.

\bibitem[{Pennington et~al.(2014)Pennington, Socher, and
  Manning}]{pennington2014glove}
Jeffrey Pennington, Richard Socher, and Christopher~D Manning. 2014.
\newblock Glove: Global vectors for word representation.
\newblock In \emph{Proceedings of the 2014 conference on empirical methods in
  natural language processing (EMNLP)}, pages 1532--1543.

\bibitem[{Phan et~al.(2019{\natexlab{a}})Phan, Thai, Hu, Jin, Sun, and
  Dou}]{phan2019scalable}
NhatHai Phan, My~T Thai, Han Hu, Ruoming Jin, Tong Sun, and Dejing Dou.
  2019{\natexlab{a}}.
\newblock Scalable differential privacy with certified robustness in
  adversarial learning.
\newblock \emph{arXiv preprint arXiv:1903.09822}.

\bibitem[{Phan et~al.(2019{\natexlab{b}})Phan, Vu, Liu, Jin, Dou, Wu, and
  Thai}]{phan2019heterogeneous}
NhatHai Phan, Minh Vu, Yang Liu, Ruoming Jin, Dejing Dou, Xintao Wu, and My~T
  Thai. 2019{\natexlab{b}}.
\newblock Heterogeneous gaussian mechanism: Preserving differential privacy in
  deep learning with provable robustness.
\newblock \emph{arXiv preprint arXiv:1906.01444}.

\bibitem[{Pinot et~al.(2019)Pinot, Yger, Gouy-Pailler, and
  Atif}]{pinot2019unified}
Rafael Pinot, Florian Yger, C{\'e}dric Gouy-Pailler, and Jamal Atif. 2019.
\newblock A unified view on differential privacy and robustness to adversarial
  examples.
\newblock \emph{arXiv preprint arXiv:1906.07982}.

\bibitem[{Ribeiro et~al.(2018)Ribeiro, Singh, and
  Guestrin}]{ribeiro2018semantically}
Marco~Tulio Ribeiro, Sameer Singh, and Carlos Guestrin. 2018.
\newblock Semantically equivalent adversarial rules for debugging nlp models.
\newblock In \emph{Proceedings of the 56th Annual Meeting of the Association
  for Computational Linguistics (Volume 1: Long Papers)}, pages 856--865.

\bibitem[{Szegedy et~al.(2013)Szegedy, Zaremba, Sutskever, Bruna, Erhan,
  Goodfellow, and Fergus}]{szegedy2013intriguing}
Christian Szegedy, Wojciech Zaremba, Ilya Sutskever, Joan Bruna, Dumitru Erhan,
  Ian Goodfellow, and Rob Fergus. 2013.
\newblock Intriguing properties of neural networks.
\newblock \emph{arXiv preprint arXiv:1312.6199}.

\bibitem[{Wu et~al.(2017)Wu, Li, Kumar, Chaudhuri, Jha, and
  Naughton}]{wu2017bolt}
Xi~Wu, Fengan Li, Arun Kumar, Kamalika Chaudhuri, Somesh Jha, and Jeffrey
  Naughton. 2017.
\newblock Bolt-on differential privacy for scalable stochastic gradient
  descent-based analytics.
\newblock In \emph{Proceedings of the 2017 ACM International Conference on
  Management of Data}, pages 1307--1322.

\end{thebibliography}

\newpage
\appendix
\section{Privacy Proof for Truncated Gumbel Mechanism}\label{sec:proofGumbel}
\begin{theorem} \label{thm:gumbel_privacy}
The truncated Gumbel perturbation mechanism, defined in Algorithm~\ref{alg:gumbel_perturbation}, is $\epsilon d_\chi$-private with respect to the Euclidean metric.
\end{theorem}
\begin{proof}
We first show for any pairs of substitutable words w and w',
$$\frac{\Pr[M(w)=u_i|K=n]}{\Pr(M(w')=u_i|K=n]}\leq\exp\biggl[\frac{2}{b}e^{\frac{2}{b}\Delta} d(w,w')\biggr],$$
where $n=|\mathcal{W}|$ and $d(w,w')=\|\phi(w)-\phi(w')\|_2$.
Conditional on $K=n$, 
$$\Pr(M(w)=u_i|K=n)=\Pr(d_i+g_i<\min_{j\ne i}d_j+g_j).$$ 
Since $g_1,\ldots,g_n$ are $i.i.d.$ random variables, we argue for each $i$ independently. 
Fix $g_{-i}=[g_1,\ldots,g_{i-1},g_{i+1},\ldots,g_n]$ as a random draw from $n-1$ independent Gumbel distributions. 
Define $g^*=\sup g: d_i + g < \min_{j\ne i}d_j + g_j.$ Then  $g_i<\min_{j\ne i}(d_j+g_j)-d_i$ if and only if $g_i\leq g^*$, which means 
$M(w)=u_i$ if and only if $g_i\leq g^*$.
Now consider another substitutable word $w'$ with a corresponding distance vector $\mathbf{d}'=[d_1',\ldots,d_n']$.
By triangle inequality, we have 
$$|d_i-d_i'|\leq d(w,w'), \textrm{ for }i=1,.\ldots,n.$$
Therefore,
\begin{align*}
    &\Pr(M(w')=u_i|K=n)\\
    =&\Pr(d_i'+g_i<\min_{j\ne i} (d_j'+g_j) )\\
    =&\Pr(g_i<\min_{j\ne i}(d_j'+g_j)-d_i')\\
 %   \leq&\Pr(g_i < \min_{j\ne i}(d_j+d(w,w')+g_j)-(d_i-d(w,w'))\\
    =&\Pr(g_i < \min_{j\ne i}(d_j+g_j)-d_i+2d(w,w')\\
    =&\Pr(g_i \leq g^*+2d(w,w')).
\end{align*}
Therefore,
\begin{align*}
&\frac{\Pr(M(w)=u_i|K=n)}{\Pr(M(w')=u_i|K=n)}\\
\geq&\frac{\Pr(g_i\leq g^*)}{\Pr(g_i\leq g^*+2d(w,w'))}\\
=&\frac{\exp(-e^{-\frac{1}{b}g^*})}{\exp(-e^{-\frac{1}{b}g^*-\frac{2}{b} d(w,w')})}\\
=&\exp[-e^{-\frac{1}{b}g^*}(1-e^{-\frac{2}{b} d(w,w')})],
\end{align*}
which is increasing in $g^*$ as 
$1-e^{-\frac{2}{b} d(w,w')}>0$.
Since $g^*\geq-2\Delta$, 
%\abhi{How do we know that $g^*$ with this bound exists?} 
%\zekun{since all $g_i$ are  bounded between $[-C,C]$, $g^*$ is at least $-\Delta-C$}
and  then
\begin{align*}
    &\frac{\Pr(M(w)=u_i|K=n)}{\Pr(M(w')=u_i|K=n)}\\
    \geq & \exp(-e^{-\frac{1}{b}(-2\Delta)}(1-e^{-\frac{2}{b} d(w,w')})) \\
    \geq &\exp\biggl[-e^{\frac{2}{b}\Delta}\cdot\frac{2}{b} d(w,w')\biggr].
\end{align*} 
By symmetry of $w$ and $w'$, we also have
$$\frac{\Pr(M(w)=u_i|K=n)}{\Pr(M(w')=u_i|K=n)}\leq\exp\biggl[\frac{2}{b}e^{\frac{2}{b}\Delta} d(w,w')\biggr].$$

Recall that $K \sim \poisson{\lambda;1,n}$. We want to show an upper bound for $\frac{\Pr(M(w)=u_i)}{\Pr(M(w')=u_i)}$, 
which is
\begin{align*}
   & \frac{\Pr(M(w)=u_i)}{\Pr(M(w')=u_i)} \\
   =&\frac{\sum_{k=1}^n\Pr(M(w)=u_i|K=k)\Pr(K=k)}{\sum_{k=1}^n\Pr(M(w')=u_i|K=k)\Pr(K=k)}\\
   \leq &\frac{\sum_{k=1}^n\Pr(M(w)=u_i|K=k)\Pr(K=k)}{\Pr(M(w')=u_i|K=n)\Pr(K=n)}\\
   \leq & \frac{n-1+\Pr(M(w)=u_i|K=n)\Pr(K=n)}{\Pr(M(w')=u_i|K=n)\Pr(K=n)},
 \end{align*}
Since 
\begin{align*}
    \Pr(M(w)=u_i|K=n)=&\exp(-e^{-\frac{1}{b}g^*})\\
    \geq&\exp(-e^{\frac{2\Delta}{b}}),
\end{align*}
and $\Pr(K=n)\geq e^{-\lambda}$ (from Definition ~\ref{def:poisson}),
\begin{align*}
   & \frac{\Pr(M(w)=u_i)}{\Pr(M(w')=u_i)} \\
   \leq & \exp\biggl(\frac{2}{b}e^{\frac{2}{b}\Delta}d(w,w')\biggr)\biggl(1 + 
   \frac{n-1}{\exp(-e^{\frac{2\Delta}{b}}-\lambda)}\biggr)\\
   =&\biggl(1+(n-1)e^{e^{\frac{2\Delta}{b}}+\lambda}\biggr)
   \exp\biggl(\frac{2}{b}e^{\frac{2}{b}\Delta}d(w,w')\biggr)\\
   \leq & 2n\exp(e^{\frac{2\Delta}{b}}+\lambda)\exp\biggl(\frac{2}{b}e^{\frac{2}{b}\Delta}d(w,w')\biggr)
\end{align*}

In order to guarantee $\epsilon$ $d_\chi$-privacy, we solve for $b$ using
\begin{equation*}
    e^{\epsilon d(w,w')}\geq 2n\exp(e^{\frac{2\Delta}{b}}+\lambda)
   \exp\biggl(\frac{2}{b}e^{\frac{2}{b}\Delta}d(w,w')\biggr).
\end{equation*}
Taking logarithm on both sides,
\begin{align*}
    &\epsilon 
    \geq \frac{1}{d(w,w')}\log_e\biggl(
2n \exp(e^{\frac{2\Delta}{b}}+\lambda) \biggr)
+\frac{2}{b}e^{\frac{2}{b}\Delta},
\end{align*}
so we need to find an upper bound for the right-hand side of the equation as a function of $b$.
\begin{align*}
    & \frac{1}{d(w,w')}\log_e\biggl(
2n \exp( e^{\frac{2\Delta}{b}}+\lambda )\biggr)
+\frac{2}{b}e^{\frac{2}{b}\Delta}\\
\leq&\frac{1}{\Delta_0}\biggl(2+
\log n  + e^{\frac{2\Delta}{b}}+\lambda \biggr)
+\frac{2}{b}e^{\frac{2}{b}\Delta} \\
=&\frac{2+\log n+\lambda}{\Delta_0} + \biggl(\frac{1}{\Delta_0}+\frac{2}{b}\biggr)e^{\frac{2}{b}\Delta},
\end{align*}
which is decreasing in $b$.
When $b\leq\Delta_0$, 
\begin{align*}
    &\frac{2+\log n+\lambda}{\Delta_0} + \biggl(\frac{1}{\Delta_0}+\frac{2}{b}\biggr)e^{\frac{2}{b}\Delta}\\
    \leq &\frac{2+\log n+\lambda}{\Delta_0} +\frac{3}{b}e^{\frac{2}{b}\Delta},
\end{align*}
it is sufficient to set
$$b = \frac{2\Delta}{W \biggl(\frac{2\Delta}{3}
\biggl(\epsilon-\frac{2+\log n+\lambda}{\Delta_0} \biggr)\biggr)},$$
where $W$ is Lambert-W function.
When $b>\Delta_0$,
\begin{align*}
    &\frac{2+\log n+\lambda}{\Delta_0} + \biggl(\frac{1}{\Delta_0}+\frac{2}{b}\biggr)e^{\frac{2}{b}\Delta}\\
    \leq &\frac{2+\log n+\lambda}{\Delta_0} +\frac{3}{\Delta_0}e^{\frac{2}{b}\Delta},
\end{align*}
it is sufficient to set
$$b = \frac{2\Delta}{\log_e \biggl(\frac{\Delta_0}{3}
\biggl(\epsilon-\frac{2+\log n+\lambda}{\Delta_0} \biggr)\biggr)}.$$
Thus, a sufficient condition for 
\begin{align*}
    &\epsilon
    \geq \frac{1}{d(w,w')}\log_e\biggl(
2n \exp(e^{\frac{2\Delta}{b}}+\lambda) \biggr)
+\frac{2}{b}e^{\frac{2}{b}\Delta},
\end{align*}
is to set $b$ to be
\begin{align*}
  \max\biggl(&\frac{2\Delta}{W \biggl(\frac{2\Delta}{3}
\biggl(\epsilon-\frac{2+\log n+\lambda}{\Delta_0} \biggr)\biggr)},\\
&\frac{2\Delta}{\log_e \biggl(\frac{\Delta_0}{3}
\biggl(\epsilon-\frac{2+\log n+\lambda}{\Delta_0} \biggr)\biggr)}
\biggr).  
\end{align*}

Now that we have proved the proposed mechanism $M$ is $\epsilon$ $d_\chi$-private with respect to Euclidean metric $d$ on a string of one word, we have for any pair of inputs $w,w'\in\mathcal{W}^\ell$ and any output $u\in\mathcal{W}^\ell$,
\begin{align*}
    &\frac{\Pr(M(w)=u)}{\Pr(M(w')=u)} = \prod_{i=1}^\ell\biggl( \frac{\Pr(M(w_i)=u_i)}{\Pr(M(w'_i)=u_i)}\biggr)\\ &\leq \prod_{i=1}^\ell\exp(\epsilon d(w_i,w_i')) = \exp(\epsilon d(w,w')),
\end{align*}
where $d(w,w')=\sum_{i=1}^\ell d(w_i, w_i')$.\qedhere
\end{proof}

For Algorithm~\ref{alg:gumbel_perturbation}, we set $\lambda = \log |\mathcal{W}|$, so that the value of $b$ used is the following:
\begin{align*}
    b =  \max\biggl(&\frac{2\Delta}{W \biggl(\frac{2\Delta}{3}
\biggl(\epsilon-\frac{2+2\log |\mathcal{W}|}{\Delta_0} \biggr)\biggr)},\\
&\frac{2\Delta}{\log_e \biggl(\frac{\Delta_0}{3}
\biggl(\epsilon-\frac{2+2\log |\mathcal{W}|}{\Delta_0} \biggr)\biggr)}
\biggr)
\end{align*}
For this value of $b$ to be defined, we must ensure that $\epsilon$ is set in a way that the logarithm and Lambert-$W$ function in the denominator has a positive argument. This holds whenever the following is true:
\begin{align*}
    \epsilon &> \frac{2 \paran{1+\log |\mathcal{W}|}}{\Delta_0}.
\end{align*}
For IMDB dataset, we have $|\mathcal{W}| = 48210$, and that for the SNLI dataset is $|\mathcal{W}| = 11673$. Using $\Delta_0 = 0.2208$ and $0.2263$ for IMDB and SNLI, respectively, the lower bounds for $\epsilon$ we obtain are $106.73$ and $91.604$, respectively.

\section{Fraction of Modified Words}
\begin{lemma} For given $\epsilon > 0$, string $x = w_1 \dots w_{\ell
}$ and any fixed $k$, the expected fraction of words that get modified using Algorithm~\ref{alg:gumbel_perturbation} is at least $(1-p)$, where $p = \exp\paran{-e^{-\frac{2\Delta}{b}}}$. In particular, $\mathbb{E}(N_w) \leq p|\mathcal{W}|$.
\end{lemma}
\begin{proof}
Fix a word $w_i \in x$. Since $u_1 = w_i$, observe that we can write the probability that it does not get modified as $\Pr\paran{\widetilde{w_i} = u_1} = \Pr\paran{g_1 < \min_{j \geq 2} \paran{d_j + g_j}}$. Let $g_1^* = \sup g: g < \min_{j \geq 2} \paran{d_j + g_j}$. Then, similar to the proof of Theorem~\ref{thm:gumbel_privacy}, $g_1 < \min_{j \geq 2} \paran{d_j + g_j}$ if and only if $g_1 \le g_1^*$. This gives $\Pr\paran{\widetilde{w_i} = u_1} = \Pr\paran{g_1 \le g_1^*} = \exp\paran{-e^{-g_1^*/b}}$. Since $g_1^* \le 2\Delta$, we can write $\Pr\paran{\widetilde{w_i}= u_1} \le \exp\paran{-e^{-\frac{2\Delta}{b}}}$. 

Thus, the expected fraction of words in $x$ that do not get modified is at most $p$, where $p = \exp\paran{-\exp\paran{-\frac{2\Delta}{b}}}$. From this, we compute the expected fraction of words that get modified as at least $(1-p)$, as desired. The bound on $\mathbb{E}(N_w)$ follows from a simple union bound over all the words in the vocabulary.
\end{proof}

Note that $\frac{\partial p}{\partial b}=\frac{\partial}{\partial b}\exp\paran{-e^{-\frac{2\Delta}{b}}} < 0$, and hence, $p$ is a decreasing function in $b$, implying that as the privacy increases ($b$ increases), the value of $\mathbb{E}(N_w)$ decreases, as expected.
%\nan{There might be a gap between the lower bound of words modified and extreme case analysis of the lower bound?}

\section{Utility Analysis vs. Sparsity of the Embedding Space}\label{sec:utility_sparsity}
We want to analyze how word substitution works for Gumbel vs. Laplace for different embedding densities. Given a word $w \in \mathcal{W}$ in the vocabulary, we let $\delta(w) = \min_{\substack{w' \in \mathcal{W}\\w \neq w'}} d(w,w')$ denote the distance to the closest word to $w$ in the embedding space. For the same value of $\epsilon$, let $n_{\textsf{Lap}} \sim \textsf{Lap}\paran{\frac{2}{\epsilon}}$ be the amount of Laplace noise added to perturb the word, and $p_{\textsf{Lap}}(w)$ be the probability that the event $\xi_w: \ \arg\min_{w'\in \mathcal{W}} \paran{||w' - \paran{w + n_{\textsf{Lap}}}||_2} = w$ (\emph{i.e.} the word remains unchanged). Then, we can compute this probability as follows:
\begin{align*}
    p_{\textsf{Lap}}(w) &= \Pr\paran{\xi_w} = \Pr\paran{||n_{\textsf{Lap}}||_2 < \delta(w)/2}\\
    &= 2\int_0^{\delta(w)/2} \frac{\epsilon}{4} e^{-\epsilon x/2} dx\\
    &= \int_0^{\epsilon \delta(w)/4} e^{-y} dy = 1 - e^{\paran{-\frac{\epsilon \delta(w)}{4}}}.
\end{align*}
Thus, as $\delta(w)$ increases (the sparsity around $w$ increases), so does $p_{\textsf{Lap}}(w)$, implying that under Laplace mechanism, words inside the sparse regions of the embedding space tend to stay unchanged. However, when $\delta(w)$ approaches $0$ (denser regions), the probability $p_{\textsf{Lap}}(w)$ vanishes. For such regions, $w$ will get modified with probability approaching one, which can potentially reduce utility.

For the same amount of $\epsilon$, the Truncated Gumbel mechanism keeps $w$ unchanged when the noise added to $w$ is smaller than any other perturbed candidate. If $p_{\textsf{Gum}}(w)$ is the probability that $w$ does not change under this perturbation, then we can write the following:
\begin{align*}
    p_{\textsf{Gum}}(w) &\geq \Pr\paran{g_1 < \delta(w) + g_2}\Pr(K \ge 2) \\
    &= \Pr\paran{g_1 - g_2 < \delta(w)} \Pr(K \ge 2)
\end{align*}
Since the difference of two i.i.d. Gumbel random variables follows a Logistic distribution, we obtain the following (by letting $G_b \sim \textsf{Logistic}\paran{0,b}$):
\begin{align*}
    p_{\textsf{Gum}}(w) &\geq \Pr\paran{G_b < \delta(w)}\Pr(K \ge 2) \\
    &= \parfrac{1}{1+e^{-\delta(w)/b}}\Pr(K \ge 2) \\
    &\geq e^{-e^{-\delta(w)/b}}\Pr(K \ge 2),
\end{align*}

where, the last inequality follows since $1+x \le e^x$. Thus, even when $\delta(w)$ approaches $0$ (denser regions), there is at least $p_{\textsf{Gum}}(w) |_{\delta(w) \to 0} \ge \frac{\Pr(K \ge 2)}{e} = \frac{1}{e}\paran{1 - \frac{\log |\mathcal{W}|}{e^{|\mathcal{W}|}}} \xrightarrow{|\mathcal{W}|\to \infty} 36.7\%$ probability that $w$ remains unchanged. This helps preserve utility by ensuring that the modified word is likely to be closer to the original word since there is a significant probability mass around the original word (specially as $|\mathcal{W}|$ increases). 

% \zekun{$w$ remains unchanged is not equivalent to the statement that the modified word is likely to be closer to the original word. So far we've only shown that when $\delta(w)$ tends to 0, probability of word remaining unchanged is positive in Gumbel but vanshies in Laplace.}\abhi{Also the caveat that $\epsilon$ does not really go to zero and hence, some words must not change.}

% \nan{I am not sure if it fair to compare laplace and truncated gumbel when their epsilon value is identical. As from the experiment aspect, they have different epsilon scales. For some epsilon value like 150 or 200, the adversarial accuracy of the laplace mechanism is almost converged to the optimal value (null privacy) but for truncated gumbel, the adversarial accuracy is quite low like $20\%~30\%$. Can we prove from another aspect, just as the $S_w$-$N_w$ plot shown in WG slides, like for the identical $N_w$ value, they have different $S_w$ distribution (or inverse)?}

\begin{table*}[t!]
\caption{Performance of adversarial training approaches on Text Data with(out) perturbations from \textbf{multivariate Laplace mechanism}. The clean accuracy of normal training is $89.50\%$ on IMDB and $82.68\%$ on SNLI. The accuracy from one model higher than that achieved by the other model in the same setting is marked by boldface.}
\label{tab:laplace_results}
\begin{tabular}{@{}cllllllllllll@{}}
% \cmidrule(r){1-5}
\toprule
\multicolumn{3}{c}{$\epsilon$}                                           & 1              & 5              & 9              & 20             & 40             & 60             & 80             & 100            & 150            & 200            \\ \midrule
\multirow{4}{*}{IMDB}                     & \multirow{2}{*}{Clean} & IBP & 81.00          & 81.00          & 81.00          & 81.00          & 81.00          & 81.00          & 81.00          & 81.00          & 81.00          & 81.00          \\\cmidrule(r){3-13}
                                          &                        & Aug & \textbf{88.22} & \textbf{88.20} & \textbf{87.34} & \textbf{87.38} & \textbf{88.60} & \textbf{88.74} & \textbf{88.12} & \textbf{88.46} & \textbf{88.00} & \textbf{87.76} \\\cmidrule(r){2-13}
                                          & \multirow{2}{*}{Adv}   & IBP & 0.30           & 0.50           & 1.20           & 4.90           & \textbf{38.60} & \textbf{68.30} & \textbf{78.50} & \textbf{80.30} & \textbf{80.90} & 81.00          \\ \cmidrule(lr){3-13}
                                          &                        & Aug & \textbf{10.80} & \textbf{8.50}  & \textbf{10.20} & \textbf{6.90}  & 9.50           & 17.70          & 32.10          & 53.00          & 80.50          & \textbf{88.30} \\ \midrule
\multicolumn{1}{l}{\multirow{4}{*}{SNLI}} & \multirow{2}{*}{Clean} & IBP & \textbf{79.19}          & \textbf{79.19}          & \textbf{79.19}          & \textbf{79.19}          & 79.19          & 79.19          & 79.19          & 79.19          & 79.19          & 79.19          \\\cmidrule(r){3-13}
\multicolumn{1}{l}{}                      &                        & Aug & 76.68 & 77.28 & 77.07 & 78.08 & \textbf{81.38} & \textbf{81.79} & \textbf{81.75} & \textbf{81.91} & \textbf{82.17} & \textbf{82.00} \\\cmidrule(r){2-13}
\multicolumn{1}{l}{}                      & \multirow{2}{*}{Adv}   & IBP & 1.84           & 1.90           & 2.21           & 3.70           & \textbf{9.22}  & \textbf{24.19} & 46.62          & 64.92          & 78.73          & 79.16          \\\cmidrule(r){3-13}
\multicolumn{1}{l}{}                      &                        & Aug & \textbf{2.44}  & \textbf{2.61}  & \textbf{3.01}  & \textbf{4.20}  & 9.14           & 24.08          & \textbf{46.94} & \textbf{66.54} & \textbf{81.44} & \textbf{81.94}\\\bottomrule
\end{tabular}
\end{table*}

\end{document}